\documentclass{llncs}
\usepackage{wrapfig}
\usepackage{bbm}
\usepackage{algorithm2e}
\usepackage{subcaption}
\usepackage{amsmath}
\usepackage{graphicx}
\usepackage{amsfonts}
\usepackage[square,numbers]{natbib}

  \renewenvironment{thebibliography}[1]{%
    \begin{oldthebibliography}{#1}%
      \setlength{\parskip}{0ex}%
      \setlength{\itemsep}{0ex}%
  }%
  {%
    \end{oldthebibliography}%
  }
 \relax

\DeclareMathOperator*{\argmax}{arg\,max}

\begin{document}


\title{A Robust UCB Scheme for Active Learning in Regression from Strategic Crowds}



%
%
%
%

%


\author{ Divya Padmanabhan \inst{1} \and Satyanath Bhat\inst{1} \and Dinesh Garg\inst{2}
 \and Shirish Shevade\inst{1} \and Y. Narahari\inst{1}}

\institute{Indian Institute of Science, Bangalore,\\
\and
IBM India Research Labs
}

\maketitle

\begin{abstract}
We study the problem of training an accurate linear regression model by procuring labels from multiple noisy crowd annotators, under a budget constraint. We propose a Bayesian model  for linear regression in crowdsourcing and use variational inference for parameter estimation. To minimize the number of labels crowdsourced from the annotators, we adopt an active learning approach. In this specific context, we prove the equivalence of well-studied criteria of active learning like entropy minimization and expected error reduction. Interestingly, we observe that we can decouple the problems of identifying an optimal unlabeled instance and identifying an annotator to label it. We observe a useful connection between the multi-armed bandit framework and the annotator selection in active learning. Due to the nature of the distribution of the rewards on the arms, we use the Robust Upper Confidence Bound (UCB) scheme with truncated empirical mean estimator to solve the annotator selection problem. This yields provable guarantees on the regret. We further apply our model to the scenario where annotators are strategic and design suitable incentives to induce them to put in their best efforts.
\end{abstract}








\section{Introduction}

Crowdsourcing platforms such as Amazon Mechanical Turk are becoming popular avenues for getting large scale human intelligence tasks executed at a much lower cost. In particular, they have been widely used
to procure labels to train learning models. 
These platforms are characterized by a large pool of diverse yet inexpensive annotators.
To leverage these platforms for learning tasks, the following issues need to be addressed:
(1) A learning model that encompasses parameter estimation and annotator quality estimation. 
(2) Identifying the best yet minimal set of instances from the pool of unlabeled data.
(3) Determining an optimal subset of annotators to label the instances.
(4) Providing suitable incentives to elicit best efforts from the chosen annotators under a budget constraint. 
We provide an end to end solution to address the above issues for a regression task.\par  
Identifying the best yet minimal set of instances to be labeled is important to minimize the generalization error, as the learner only has limited budget.
This involves selection of those unlabeled instances, the labels of which when fed to the learner, yield  maximum performance enhancement of the 
underlying model. The question of choosing an optimal set of unlabeled examples occupies center
stage in the realm of active learning. 
Past work on active learning in crowdsourcing apply to classification ~\cite{rodrigues2014,  RaykarAISTATS14} and most of these do not directly apply to regression where the space of labels is unbounded. For instance, the Markov Decision Processes (MDP) based method \cite{RaykarAISTATS14} 
relies on label space and thereby the state space being finite, which is not the case in regression.\par
Similar to the instance selection problem, the annotator choice to label an instance also has a bearing on the accuracy of the learnt model. Optimal annotator selection, in the context of classification, has been addressed using multi-armed bandit (MAB) algorithms \cite{AbrahamAKS13}.
Here the annotators are considered as the arms and their qualities as the stochastic rewards. In classification, the quality of the annotators is
modeled as a 
Bernoulli random variable, thereby making it suitable for application of algorithms such as UCB1 \cite{UCBAuer2002, Bubeck2012a}. However for regression tasks, the labels provided by the annotators
are naturally modeled to have Gaussian noise, the variance of which is a measure of the quality of the annotator. This 
in turn is a function of the effort put in. Therefore, optimal annotator set selection problem involves identifying annotators with low variance.
Though existing work  has adopted MAB algorithms for estimating variance \cite{Neufeld14} and several other applications \cite{Sen2015}, there is a research gap in its applicability to active learning and regression tasks
and in particular where heavy tailed distributions arise as a result of squaring the Gaussian noise. To bridge this gap, we invoke ideas from 
Robust UCB \cite{Bubeck2012b} and set up  theoretical guarantees for annotator selection in active learning. 

Another non-trivial challenge emerges when we are required to account for the  strategic behavior of the human agents.
An agent, in the absence of suitable incentives, may not find it beneficial to put in efforts while labeling the data.
To induce best efforts from agents, the learner could appropriately incentivize  
them. In the field of mechanism design, several incentive schemes exist \cite{Dayama2015, TranThanh2014}. To the best of our knowledge, 
such schemes have not been explored
in the context of active learning for regression. 
\\\\
\textbf{Contributions}:
The key contributions of this paper are as follows.\\
(1)\textbf{Bayesian model for Regression}: In Section \ref{sec:bayesian-lr}, we set up a novel   Bayesian model for regression using labels from multiple annotators with varying noise levels, which makes the problem  challenging. 
We use variational inference for parameter estimation
to overcome intractability issues.\\ 
(2)\textbf{Active learning for crowd regression and decoupling instance selection and annotator selection}: In Section \ref{sec:criteria-al}, we focus on various active learning criteria as applicable to the proposed regression model.  
Interestingly, in our setting, we show that the  criteria of {\em minimizing estimator error} and {\em minimizing estimator's entropy} 
are equivalent.
These criteria also remarkably enable us to  decouple the problems of
instance selection and annotator selection. \\
(3)\textbf{Annotator selection with multi-armed bandits}:
 In Section \ref{sec:al-annotator}, we describe the problem of selecting an
annotator having least variance. We establish an interesting connection of this problem to
the multi-armed bandit problem. In our formulation, we
work with the square of the label noise to cast the problem
into a variance minimization framework; the square of the
noise follows a sub-exponential distribution. We show that
standard UCB strategies based on $\psi$-UCB \cite{Bubeck2012a} are not applicable
and we propose the use of robust UCB \cite{Bubeck2012b} with truncated empirical mean. We show that the logarithmic regret bound of robust UCB is
preserved in this setting as well. Moreover the number of samples discarded is also logarithmic.\\
(4)\textbf{Handling strategic agents}: In Section \ref{sec:payment_strategic}, we consider the case of strategic annotators where the learner needs to induce them to put in their best efforts. For
this, we propose the notion of `\emph{quality compatibility}' and introduce a payment scheme that induces agents to put in their
best efforts and is also individually rational. \\
(5)\textbf{Experimental validation}: We describe our experimental findings in Section \ref{sec:experiments}. We compare the RMSE
and regret of our proposed models with state-of-the-art benchmarks on several real world datasets. Our experiments demonstrate a superior performance.
\section{Related Work}
A rich body of literature exists in the field of active learning for statistical models where labels are provided by a single source
\cite{Roeder12,Cohn1996,Burbidge2007,CaiICDM2013}. Popular techniques include minimizing the variance or uncertainty of the learner, 
query by committee schemes \cite{Seung1992}
and expected gradient length \cite{SettlesNIPS2007} to name a few.
In the literature on Optimal Experimental Design in Statistics,
the selection of most informative data instances is captured
by concepts such as A-optimality, D-optimality, etc. \cite{FedorovOED,varunThesis}. The idea is to construct confidence regions for the learner 
and bound these regions. 
A survey on active learning approaches for  various problems is presented in
 \cite{Settles10activelearning}.
\par The works that have looked into active learning for  regression are applicable only  for a single noisy source, and not to a  crowd. In crowdsourcing,
several learning models for regression have been proposed, for instance, \cite{ Raykar2012JMLR, Ristovski2010} obtain the maximum likelihood estimate 
(MLE) and maximum-a-posteriori (MAP) estimate respectively.
\cite{groot} proposes a scheme to aggregate information from multiple annotators for regression  
using Gaussian Processes. \cite{BiUAI2014, Raykar2009} develop models for classification using crowds. However, these do not employ techniques from active learning. Also, they do not obtain a posterior distribution 
over the parameters,
and hence do not perform probabilistic inference.
Of late, there have been a few crowdsourcing classification models employing the active learning paradigm \cite{rodrigues2014, Zhao13, WauthierNIPS2011,RaykarAISTATS14,Dekel2012}. 
These include uncertainty-based methods and MDPs.
To the best of our knowledge, active learning for regression using the crowds has not been looked at explicitly.
\par  When an annotator is requested to label an instance, and the annotator, being strategic, does not put
in the best effort, the learning algorithm could seriously underperform. So we must incentivize the annotator to
induce the best effort. 
Such studies are not reported in the current literature.
~\cite{papadimitriou, Dekel2010759} propose payment schemes for linear regression  
for crowds.  Both  ~\cite{papadimitriou, Dekel2010759} make the assumption that an instance is provided only to a single annotator
and also do not look at the active learning paradigm.
The idea in our work is to design incentives for active learning in the context of crowdsourced regression which would
induce the annotators to put in their best efforts.

In the next section, we explain our model for regression using the crowd, assuming non-strategic annotators.
\section{Bayesian Linear Regression from a Non-strategic Crowd}
\label{sec:bayesian-lr}
Given a data instance $\textbf{x}\in {\mathbb{R}}^d$, the linear regression model aims at predicting its 
label  
$y$ such that $y = \textbf{w}^\top \textbf{x}$. Instead of  
$\textbf{x}$, non-linear functions $\Phi(.)$ of $\textbf{x}$, 
can be used. To avoid notational clutter, we work with $\textbf{x}$ throughout this paper. The coefficient vector $\textbf{w} \in \mathbb{R}^d$ is unknown and training a linear regression model essentially 
involves finding $\textbf{w}$. Let  $\mathcal{D}$ be the initially procured training dataset and let $\mathcal{U}$ denote  the pool of unlabeled instances. We later (in Section \ref{sec:criteria-al}) select instances from $\mathcal{U}$ 
via active learning to enhance our model.
\par
In classical linear regression, the labels are assumed to be provided by a single noisy source.
In crowdsourcing, however, there are multiple annotators denoted by the set $S = \{1, \ldots, m\}$.
 Each of the annotators provides a label vector which we denote by $\textbf{y}_1, \ldots, \textbf{y}_m$,
where $\textbf{y}_j \in {\mathbb{R}}^n$ for $j = 1 ,\ldots, m$. 
Each annotator may or may not provide the label for every instance in the training set. We, therefore, define an indicator matrix 
$I \in \{ 0,1\}^{n \times m}$, where $I_{ij} = 1$ if annotator $j$ labels instance $\textbf{x}_i$, else $I_{ij} = 0$. 
We denote by $n_j$, the number of labels provided by annotator $j$. That is, $n_j = \sum_{i} I_{ij}$. We also define a matrix 
$\textbf{X}^j \in \mathbb{R}^{n_j \times d}$ whose rows contain the instances that are labeled by annotator $j$. Also, we denote by $y_{ij}$,
the label provided by annotator $j$ for $\textbf{x}_i$, which is the same as  $i^{\text{th}}$ element of the label vector $\textbf{y}_j$. 
The true label of a data instance $\textbf{x}_i$ is given by $\textbf{w}^\top\textbf{x}_i$. Each annotator $j$ introduces 
a Gaussian noise in the label he provides. That is,
$ y_{ij} \sim \mathcal{N}(\textbf{w}^\top\textbf{x}_i,\beta_j^{-1})$
where, $\beta_j$ is the precision or inverse variance of the distribution followed by $y_{ij}$. Intuitively, $\beta_j$ is directly proportional to the
effort put in by annotator $j$. We assume that there is always a maximum level of effort that annotator $j$ can put in and 
inverse variance corresponding to his best effort is given by $\beta_j^*$, which is unknown to the learner as well as other annotators.\\
In general, an annotator may be strategic and may exert a lower effort level $\beta_j < \beta_j^*$ if 
appropriate incentives are not provided. In this section, however, we adhere to the assumption that annotators are non-strategic and 
annotator $j$ always introduces a precision of $\beta_j^*$, thereby setting $\beta_j = \beta_j^*$. The parameters of the linear regression model
from crowds, therefore, become $\Theta = \{\textbf{w} ,  \beta_1 , \cdots, \beta_m  \}$. The aim of training a
linear regression model is to obtain estimates of  $\Theta$ using the training data $\mathcal{D}$. 
We now describe a Bayesian framework for this.

 \begin{figure}[h] 
        \centering
        \includegraphics[width = 0.5\textwidth]{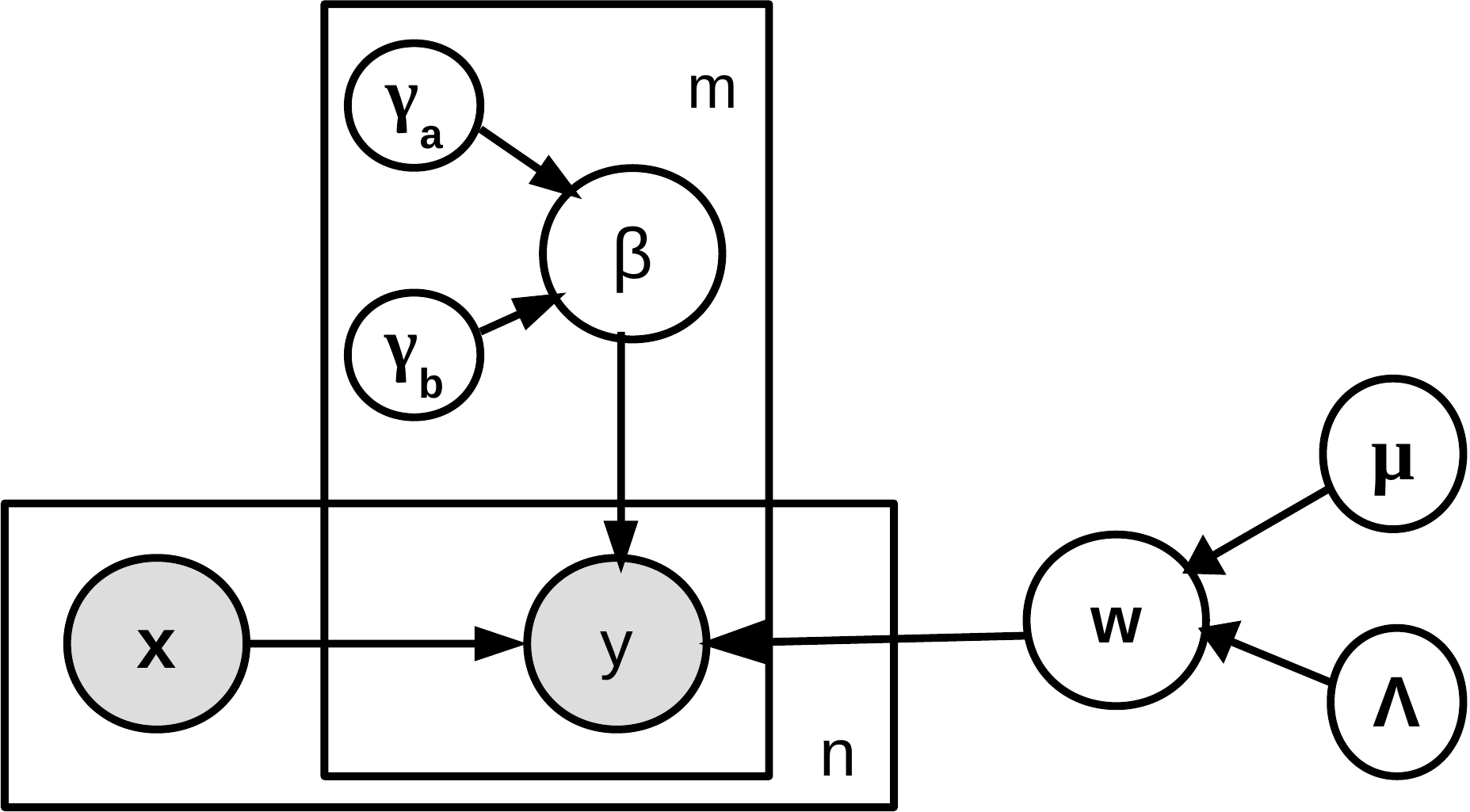}
        \caption{Plate notation for our model }
         \label{fig:plate-notation}
  \end{figure}
\noindent
\textbf{Bayesian Model and Variational Inference for Parameter Estimation}: 

 A Bayesian framework for parameter estimation is well suited for active learning as incremental learning can be done conveniently. 
 Bayesian framework has been developed for estimating the parameters of the linear regression model when labels of training data 
 are supplied by a single noisy source~\cite{prml_bishop}. To the best of our knowledge, the counterpart of such a Bayesian framework 
 in the presence of multiple annotators has not been explicitly explored. 
 \setlength{\intextsep}{0.2pt}
 We assume a Gaussian prior for $\textbf{w}$ with mean $\mu_0$ and precision matrix or inverse covariance 
matrix $\Lambda_0$.
We assume Gamma priors for $\beta_j$'s, that is, 
 $p(\textbf{w}) \sim \mathcal{N}(\mu_0, \Lambda_0^{-1}), 
 p(\beta_j) \sim \mathcal{G}(\gamma_{a0}^j, \gamma_{b0}^j) \text{ for } j=1,\ldots, m $.
 The plate notation of the Bayesian model described above is provided in Figure
\ref{fig:plate-notation}.
The computation of the posterior distributions
$p(\textbf{w} \mid \mathcal{D})$ and $p(\beta_j \mid \mathcal{D}) $ for $j= 1, \cdots, m$  is not tractable. Therefore, 
we appeal to variational approximation methods \cite{BealThesis2003Variational}. These methods approximate the posterior distributions using mean field assumptions.
We use $q(\bold{w})$ and $q(\beta_j)$ to represent the mean field variational approximation of  $p(\textbf{w}\mid\mathcal{D})$ and
$p(\beta_j \mid \mathcal{D})$ respectively.
The variational approximation begins by initializing the parameters of the prior distributions,
$\{ \mu_0, \Lambda_0\}$ and
$\{ \gamma_{a0}^j, \gamma_{b0} ^j\}$ for all $j =1, \cdots, m$.
At each iteration of the algorithm, the parameters of the posterior approximation are updated and the steps are repeated until convergence.

\begin{lemma}
The variational update rules for the posterior approximations using mean field assumptions are
$
 q(\bold{w}) \sim \mathcal{N}(\mu_n, \Lambda_n^{-1})
$ and $
 q(\beta_j) \sim \mathcal{G}( \gamma_{an}^{j}, \gamma_{bn}^{j} )
$
 where
 \allowdisplaybreaks
 \begin{align}
\Lambda_n &=  \left[\Lambda_0 + \sum\nolimits_j \mathbb{E}[\beta_j] \sum\nolimits_{{i : I_{ij} = 1}} \textbf{x}_i \textbf{x}_i^\top \right] \label{lambda_w_var_upd}\\
\mu_n &= \Lambda_n^{-1}\left[\Lambda_0 \mu_0 +  \sum\nolimits_j \mathbb{E}[\beta_j] \sum\nolimits_{{i : I_{ij} = 1}} y_{ij}\textbf{x}_i\right] \label{mu_w_var_upd}\\
\gamma_{an}^{j} &= \gamma_{a0}^{j} + {n_j}/{2} \label{gamma_a_var_upd}\\
\gamma_{bn}^{j} &= 
\gamma_{b0}^{j} + \frac{1}{2}  \sum\nolimits_{{i : I_{ij} = 1}} \left( y_{ij}^2 - y_{ij}\mu_n^ \top \bold{x}_i \right) \nonumber \\
 &\quad +  \frac{1}{2} Tr\left({\textbf{X}^j}^\top \textbf{X}^j\Lambda_n^{-1}\right) + \frac{1}{2}\mu_n^\top {\textbf{X}^j}^\top\textbf{X}^j\mu_n   \label{gamma_b_var_upd}
 \end{align}
\end{lemma}
\begin{proof}
If $p$ and $q$ denote the true and approximate posterior joint distributions of the parameters respectively, we know that,
$\ln p(\mathcal{D}) =  \mathcal{L}(q)+ KL(q \mid\mid p)$
, where, $
\mathcal{L}(q)= \int q(\Theta) \ln \left\lbrace \frac{p(\mathcal{D}, \Theta)}{q(\Theta)} \right\rbrace d\Theta
$ and $KL(q||p) = - \int q(\Theta) \ln \left[ \frac{p(\Theta \mid \mathcal{D})W}{q(\Theta)}\right] d\Theta
$ is the KL divergence between the distributions $q$ and $p$. 
By the mean field assumption, the joint distribution $q(\bold{w}, \beta_1, \cdots, \beta_m)$
factorizes as follows, $ q(\bold{w}, \beta_1, \cdots, \beta_m) $ $= q(\bold{w})\prod_{j=1}^{m} q(\beta_j)$.
For simplicity we denote by $q_{\bold{w}}$ the distribution $q(\bold{w})$ and by $q_{\bold{\beta_j}}$ the distribution
 $q(\beta_j)$.
\begin{align}
\label{eqn:l-q}
\mathcal{L}&(q) = \int q_{\bold{w}} \prod_{i=1}^m q_{\beta_j}
\left\lbrace \ln p(\mathcal{D}, \Theta) - \sum_{i \in {w,\beta_1,\cdots,\beta_m}} q_i \right\rbrace d\beta_j d\bold{w}\nonumber \\
&\propto \int q_\bold{w} \left \lbrace \int \ln  p(\mathcal{D}, \Theta)
\prod_{i=1}^{m} q_{\beta_i} d\bold{\beta_i} \right\rbrace d\bold{w} - \int q_{\bold{w}} \ln q_{\bold{w}} d\bold{w} \nonumber\\
& =\int q_\bold{w} \ln \tilde{p}(\mathcal{D}, \bold{w}) d\bold{w} - \int q_{\bold{w}} \ln q_{\bold{w}} d\bold{w}
\end{align}
where,
$\tilde{p}(\mathcal{D}, \bold{w}) = E_{\beta}
\left[\ln p(\mathcal{D}, \Theta) \right] + \;\;\text{constant}$ and
$\beta = \{\beta_1, \ldots, \beta_m\} $.
In order to minimize $KL(q || p)$, we must maximise
 $\mathcal{L}(q)$. Eqn (\ref{eqn:l-q}) shows that $\mathcal{L}(q)$
is the negative KL-divergence between $\tilde{p}(\mathcal{D}, \bold{w})$ and $q_{\bold{w}}$. $\mathcal{L}(q)$ is maximised 
when the KL-divergence between $\tilde{p}(\mathcal{D}, \bold{w})$ and $q_{\bold{w}}$ is minimized. Therefore, we must set
$q_{\bold{w}} =\tilde{p}(\mathcal{D}, \bold{w}) =  E_{\beta}
\left[\ln p(\mathcal{D}, \Theta) \right]$.
By similar calculations, we must set,
$q_{\beta_j} =\tilde{p}(\mathcal{D}, \beta_j) =  E_{\bold{w},\beta_{-j}}
\left[\ln p(\mathcal{D}, \Theta) \right]$, where
$\beta_{-j} = \{\beta_1, \cdots, \beta_{j-1}, \beta_{j+1},
\cdots, \beta_m\}$.
\allowdisplaybreaks
\begin{align*}
\log\; & q(\bold{w}) \propto E_{\beta} \left[ \log p(Y, \bold{w}, \beta\mid \textbf{X}, \bold{\delta}, \bold{\gamma})\right] 
\nonumber \\ &= E_{\beta} \left[ \log p( \bold{w}, \beta ) + \log p(Y\mid\textbf{X}, \bold{w}, \beta)\right] \nonumber \\
&=\log p(\bold{w} \mid \bold{\delta}) + E_\beta \left[ \log p(\beta\mid \gamma) \right] 
+  E_{\beta} \left[
\log p(Y\mid\textbf{X}, \bold{w}, \beta)\right]  \nonumber \\
&\propto \frac{1}{2\pi} | \Lambda_0 |- \frac{1}{2} (\bold{w} - \mu_0)^\top \Lambda_0  (\bold{w} - \mu_0) \\
&\;\;\;+ E_{\beta}\left[ \log \prod_{ij} \frac{\beta_j}{2\pi} \exp \left( -\frac{\beta_j (y_{ij} - \bold{w}^\top x_i)^2}{2} \right)\right] \nonumber \\
&\propto - \frac{1}{2} (\bold{w} - \mu_0)^\top \Lambda_0  (\bold{w} - \mu_0) \\
&\;\;\;+  E_{\beta}\left[ \sum_{ij}\log  \frac{\beta_j}{2\pi} - \left( \frac{\beta_j (y_{ij} - \bold{w}^\top x_i)^2}{2} \right)\right]  \nonumber \\
 &= - \frac{1}{2} (\bold{w} - \mu_0)^\top \Lambda_0  (\bold{w} - \mu_0) \\
 &\;\;\;+
  \sum_{ij}E_{\beta_j}\left[\log  \frac{\beta_j}{2\pi} - \left( \frac{\beta_j (y_{ij} - \bold{w}^\top x_i)^2}{2} \right)\right] \nonumber \\
  & \propto - \frac{1}{2} (\bold{w} - \mu_0)^\top \Lambda_0  (\bold{w} - \mu_0) \\
 & \;\;\;+
  \sum_{j}E_{\beta_j}\left[ - \left( \frac{\beta_j \sum_i (y_{ij} - \bold{w}^\top x_i)^2}{2} \right)\right] \nonumber \nonumber \\
  &= - \frac{1}{2} (\bold{w}^\top \Lambda_0\bold{w} - 2\bold{w}^\top \Lambda_0\mu_0 + \mu_0^\top \Lambda_0\mu_0 \\
  & \;\;\; +
 \sum_{ij} \frac{E_{\beta_j}\left[ \beta_j\right]}{2}(y_{ij}^2 + x_i^\top \bold{w}\bold{w}^\top x_i - 2y_{ij} \bold{w}^\top x_i)\nonumber \\
  &\propto - \frac{1}{2} \bold{w}^\top \left[\Lambda_0 + \left(\sum\nolimits_jE[\beta_j]  \sum\nolimits_{{i:I_{ij} = 1}}  \textbf{x}_i\textbf{x}_i^\top
  \right)\right]\bold{w} \\
  &\;\;\;
  + \bold{w}^\top \left[\Lambda_0\mu_0 - \frac{1}{2}  \sum\nolimits_jE[\beta_j]  \sum\nolimits_{{i:I_{ij} = 1}}  y_{ij}\bold{x}_i\right]
  \end{align*}
   By completing the squares we get the update rules for $\bold{w}$.
   The similar steps can be performed to get the variational updates for $\beta_j$. Due to constraints on space, we have not included the steps.
  \end{proof}

 The variational updates for $\mu_n$ and $\Lambda_n$ defined in  Eqns (\ref{lambda_w_var_upd}) and (\ref{mu_w_var_upd}) involve 
 $\mathbb{E}[\beta_j] = \gamma_{an}^{j}/\gamma_{bn}^{j} $.
 The updates for $\gamma_{bn}^{j}$ given in Eqn (\ref{gamma_b_var_upd}) 
 involve $\mu_n$ and $\Lambda_n$. This interdependency between the update
 equations leads to an iterative algorithm.
 \begin{remark}[Parameter Estimation]: Our approach is not tied to the variational inference approximation scheme. For example, MCMC can be used instead.
 \end{remark} 
 
  \begin{lemma}
\label{lemma:var_upd_aysmptotics}
\textbf{Asymptotic convergence of Bayes estimators}:
 Let $\textbf{w}^*$ be the true underlying value of $\textbf{w}$ and the Bayes estimator for $\textbf{w}$ under the least squares loss be $\mu_n$.
Then, $ \lim_{n \to \infty}  \mathbb{E}_{\mathcal{D}}[\mu_n] \rightarrow \textbf{w}^*$.
\end{lemma}
\begin{proof}
Let $\mu_n$ and $\Lambda_n$ be the mean and precision respectively, of the approximate posterior distribution $q(w)$ , estimated from the
training set $\mathcal{D}$. Let $\bold{w^*}$ be the realized value of the underlying $\bold{w}$.
\begin{align}
\label{lambda_n_exp}
 E_\mathcal{D}[\mu_n] &= \Lambda_n^{-1}(\Lambda_0 \mu_0 + \sum\nolimits_j E[\beta_j] \sum\nolimits_{{i:I_{ij} = 1}} \bold{x}_i E_\mathcal{D} [y_{ij}]) \nonumber\\
  &= \Lambda_n^{-1}(\Lambda_0 \mu_0 + (\Lambda_n - \Lambda_0)) \bold{w^*} \nonumber \\ 
  &= \bold{w}^* + \Lambda_n^{-1}\Lambda_0(\mu_0 - \bold{w}^*)
\end{align}
If the second term in Eqn \ref{lambda_n_exp} approaches $0$ as $n \to \infty$, the estimate $\mu_n$ is an asymptotically unbiased
estimate for $\bold{w}$.
Using standard linear algebra results, we can prove that the determinant of the precision matrix $\det(\Lambda_n)$ approaches $\infty$ with
large number of samples, that is, $\lim_{n \rightarrow \infty} \det(\Lambda_n) \rightarrow \infty$. Hence the second term in 
Eqn \ref{lambda_n_exp} approaches 
zero. Therefore $\lim_{n \rightarrow \infty} E_\mathcal{D}[\mu_n] \rightarrow \bold{w}^* $.
\end{proof}
Lemma \ref{lemma:var_upd_aysmptotics} is a desirable property of the estimators, and in general
holds true for Bayes estimators. 

\subsubsection*{Inference:}
\label{sec:inference} 
We now describe an inference scheme to make prediction about the label of a test data instance.
We denote by $\widehat{y}_{\text{test}}$ the predicted label for the test instance $\textbf{x}_{\text{test}}$.
From the Bayesian framework of parameter estimation, 
The posterior predictive distribution for $\widehat{y}_{\text{test}}$ turns out to be as follows: $p(\widehat{y}_{\text{test}}\mid \textbf{x}_{\text{test}}, \mathcal{D})
\sim \mathcal{N}(\textbf{x}_{\text{test}}^\top\mu_n, \textbf{x}_{\text{test}}^\top \Lambda_{n}^{-1} \textbf{x}_{\text{test}})$.
This follows from standard results in \cite{prml_bishop}.
We can use this distribution later in scenarios like active learning.

\section{Active Learning for Linear Regression from the Crowd}
\label{sec:seq-estimation}
We now discuss various active learning  \cite{Settles10activelearning} strategies in our framework.
Let $\mathcal{U}$ be the set of unlabeled instances.
The goal is to identify an instance, say $\textbf{x}_k \in \mathcal{U}$, for which seeking a label and retraining the model with this
additional training example will improve the model in terms of the generalization error. In the crowdsourcing context, since multiple annotators are involved, 
we also need to identify the annotator $t$ from whom we should obtain the label for $\textbf{x}_k$. The active learning criterion, thus,
involves finding a pair $(k,t)$ so that retraining with the new labeled set $\mathcal{D} \cup \{(\textbf{x}_k, y_{kt})\}$ would provide maximum improvement 
in the model.
\subsection{Instance Selection}
\label{sec:criteria-al}
 To our crowdsourcing model, we now apply  two   criteria well-studied in active learning from a single source.  
 We also show that all these seemingly different criteria embody the same logic. 
\subsubsection{Minimizing Estimator Error}
\label{point:estimator_error}
Minimizing estimator error is a natural criterion for active learning  \cite{Roy2001}. 
The error in the estimator $\mu_{n+1}$, if we choose a pair $(k,t)$, is given by,
$\text{Err}(\mu_{n+1}) = \mathbb{E}_{y_{kt}} \left[\mu_{n+1} \right] - \textbf{w} $. The error in the estimator $\mu_{n}$, before including the instance $(\textbf{x}_k, y_{kt})$ in the training set is,
$\text{Err}(\mu_{n}) = \mu_{n} - \textbf{w} $.
\begin{lemma}
\label{lemma:min-est-error}
The relation between errors in $\mu_{n+1}$ and $\mu_n$ is given by,
\begin{eqnarray}
\left \Vert \text{Err}(\mu_n)\right \Vert /(1 + \beta_t \textbf{x}_k^\top \Lambda_n^{-1}\textbf{x}_k )\leq \left \Vert \text{Err} (\mu_{n+1}) \right \Vert  \leq  \left \Vert \text{Err}(\mu_n)\right \Vert
  \label{mu_w_error_bound}
\end{eqnarray}
\end{lemma}
\begin{proof}

We first compute $E_{y_{kt}}[ \mu_{n+1}]$. 
\begin{align}
E_{y_{kt}} \left[ \Lambda_{n+1}\mu_{n+1} \right] &= \Lambda_{n+1} E_{y_{kt}} \left[\mu_{n+1}\right] \nonumber  \\
&= \Lambda_n\mu_n + \textbf{x}_k (\textbf{x}_{k}^\top\textbf{w}  )\beta_t
\end{align}
Making necessary substitutions and rearranging the terms, 
\begin{align*}
 E_{y_{kt}} \left[\mu_{n+1}\right] -\mu_n  = - \Lambda_n^{-1}\textbf{x}_k  \textbf{x}_k ^\top \text{Err}(\mu_{n+1}) \beta_t
\end{align*}
Again rearranging the terms and subtracting $\textbf{w}$ from both the sides yields,
%
$\text{Err}(\mu_{n+1})  = \left(I + \Lambda_n^{-1}\textbf{x}_k  \textbf{x}_k ^\top \beta_t \right)^{-1} \text{Err}(\mu_n)$.
We now bound  $\text{Err}(\mu_{n+1})$,  in terms of the old error, 
$\text{Err}(\mu_n)$ as follows:
$
\left \Vert \text{Err} (\mu_{n+1}) \right \Vert \leq \left \Vert(I + \Sigma_{n} \textbf{x}_k \textbf{x}_k^\top \beta_t)^{-1}\right \Vert \left \Vert \text{Err}(\mu_n)\right \Vert
$
where, $\left \Vert(I + \Lambda_n^{-1} \textbf{x}_k \textbf{x}_k^\top \beta_t)^{-1}\right \Vert$ is the spectral norm of the matrix 
$(I + \Lambda_n^{-1} \textbf{x}_k \textbf{x}_k^\top \beta_t)^{-1}$. Since $\Lambda_n^{-1} \textbf{x}_k \textbf{x}_k^\top$ is a rank one matrix, 
the matrix $I + \Lambda_n^{-1} \textbf{x}_k \textbf{x}_k^\top \beta_t$ has $d-1$ eigenvalues equal to 1 and one eigenvalue equal to 
$1 + \beta_t \textbf{x}_k^\top \Lambda_n^{-1} \textbf{x}_k$. Note, $ \textbf{x}_k^\top \Lambda_n^{-1} \textbf{x}_k > 0$ since $\Lambda_n^{-1}$ is a
positive definite matrix. Therefore, spectral norm of the matrix $(I + \Lambda_n^{-1} \textbf{x}_k \textbf{x}_k^\top \beta_t)^{-1} $ is $1$ and
its minimum eigenvalue is $1/(1 + \beta_t \textbf{x}_k^\top \Lambda_n^{-1} \textbf{x}_k)$ and we  arrive at the error bound.
\end{proof}
From Theorem\nobreakspace \ref {lemma:min-est-error}, it is clear that to reduce the value of the lower bound, we must pick a pair $(k,t)$  for which the score 
$\beta_t \textbf{x}_k^\top \Lambda_n^{-1} \textbf{x}_k$ is maximum. 
\subsubsection{Minimizing Estimator's Entropy}
\label{point:estimator_entropy}
This is another natural criterion for active learning which suggests that the entropy of the estimator after adding an
example should decrease  \cite{Lindley:OnMeasureInformationExperiment,MacKay_information-basedobjective}. Formally, 
let $H(\textbf{w} \mid \mathcal{D})$ and $H(\textbf{w}\mid \mathcal{D'})$ denote the entropies of the estimator before and after adding an example, respectively, 
where we have $\mathcal{D'} = \mathcal{D} \cup \{(\textbf{x}_k, y_{kt})\}$. Again, let us assume $\beta_j$'s are known for the time being. 
The entropy of the distribution before adding an example satisfies: $H(\textbf{w}\mid \mathcal{D}) \propto \det( \Lambda_n^{-1})$. After adding the example, entropy function behaves as follows.
$H(\textbf{w}\mid \mathcal{D'}) \propto  \det(\Lambda_{n+1}^{-1})$, where 
\begin{align}
\label{entropy-exact-relation}
\det(\Lambda_{n+1}^{-1}) &= \det(\Lambda_n^{-1})/(1 + \beta_t \textbf{x}_k^\top \Lambda_n^{-1} \textbf{x}_k)
\end{align}
From  (\ref{entropy-exact-relation}), we would like to
choose an instance $\textbf{x}_k$ and an annotator $t$ that jointly maximize $ \beta_t \textbf{x}_k^\top \Lambda_n^{-1} \textbf{x}_k$
so that $\det(\Lambda_{n+1}^{-1})$ as well as estimator's entropy are minimized. Recall, the same selection strategy was obtained while using the
 {\em minimize estimator error} criterion. 
Let $\lambda^{*}=\lambda_{\text{max}}( \Lambda_n^{-1})$ and $\lambda_{*}=\lambda_{\text{min}}( \Lambda_n^{-1})$. We can further bound the estimator precision as follows. 
\[ 1/(1 + \beta_t  \lambda^{*}\left \Vert\textbf{x}_k\right \Vert^2 ) \leq \det(\Lambda_{n+1}^{-1})/\det(\Lambda_{n}^{-1}) \leq 1/(1 + \beta_t  \lambda_{*}\left \Vert\textbf{x}_k\right \Vert^2)  \]
We observe that the selection of the best instance $\textbf{x}_k$ and the best annotator $S_t$ can be decoupled. That is, we can first select an instance 
$\textbf{x}_k$ for which $\textbf{x}_k^\top \Lambda_n^{-1} \textbf{x}_k$ is maximum and independently select an annotator for whom $\beta_t$ is maximum.
But this scheme of annotator selection
may lead to starvation of best annotators if 
the annotators have not been explored sufficiently.  Hence we only use
this strategy for selecting an instance and not for selecting the annotator. 
\subsection{Selection of an Annotator}
 \label{sec:al-annotator}
 Having chosen the instance $\textbf{x}_k$, next the learner must decide which annotator should label it.  Consider any arbitary sequential selection algorithm $\emph{A}$ for the annotators.
 If the variance of the annotators' labels were known upfront, the best strategy would be to always select the annotator introducing  the minimum variance 
$1/\beta^* = \min_{1\leq j \leq m} 1/\beta_j$.  The variances
of the annotators' labels are unknown and hence a sequential selection algorithm $\emph{A}$
incurs a regret defined by Regret-Seq($\emph{A}$) below. 
We denote
the sub-optimality of annotator $j$ by $\Delta_j = (1/\beta_j) - (1/\beta^*)$.

\begin{definition}{\textbf{Regret-Seq(A,{\tiny{ }}t):}}
 If $T_j(t)$ is the number of times annotator $j$ is selected in
$t$ runs of $\emph{A}$, the expected regret of $\emph{A}$ in $t$ runs, with respect to the choice of annotator, is computed as, 
 $\text{Regret-Seq}(\emph{A},t) = \sum_{j=1}^{m} \Delta_j \mathbb{E}[T_j(t)]$.
\end{definition}
The problem is to formally establish an annotator selection strategy which yields a regret as low as possible. 
The main challenge is that the annotators' noise level is 
unknown and must be estimated simultaneously while also deciding on the selection strategy. 
We observe the connections of this problem to the multi-armed bandit (MAB) problem. 
In MAB problems, there are $m$ arms each producing rewards from fixed distributions
$P_1, \cdots, P_m$ with unknown means $\gamma_1, \cdots, \gamma_m$. The goal is to maximise the overall reward and for this, at every time-step a decision has to be made as to which arm must be pulled.
We denote the sub-optimality of arm $i$ by $\Delta_i^{\text{MAB}} = \gamma_* - \gamma_i$, where $\gamma^* = \max_{1 \leq i \leq m} \gamma_i$. 

\begin{definition}{\textbf{Regret-MAB(M,{\tiny{ }}t):}}
 If $T_i(t)$ is the number of times arm $i$ is selected in
$t$ runs of any MAB algorithm $M$, the expected regret of $M$ in $t$ runs, Regret-MAB($M , t$), is computed as,
 $\text{Regret-MAB}(M, t) = \sum_{i=1}^{m} \Delta_i^{\text{MAB}} \mathbb{E}[T_i(t)]$.
\end{definition}
We now show that the active learning problem in crowdsourcing regression tasks can be mapped to the MAB problem. 
We know that,
$\mathbb{E}[(y_{kj}-(\textbf{w}^\top \textbf{x}_k))^2] =  1/\beta_j$. 
Since we are interested in the annotator introducing the minimum variance,
we could work with a MAB framework where the rewards of the arms (annotators in our case) are drawn from the distribution of $-(y_{kj}-(\textbf{w}^\top \textbf{x}_k))^2$.
This idea was used in \cite{Neufeld14} in the context of sequential selection from a pool of Monte Carlo estimators. 
If the selection strategy $A$ appeals to  any MAB algorithm $M$ defined on the
distributions $-(y_{kj}-(\textbf{w}^\top \textbf{x}_k))^2$,  Regret-MAB($M, t$) will be the same as Regret-Seq($A, t$), as proved
by \cite{Neufeld14}.
This implies that for the selection strategy, 
we could work with any standard MAB algorithm such as UCB on the distribution of $-(y_{kj}-(\textbf{w}^\top \textbf{x}_k))^2$ and 
 $\text{Regret-Seq(\emph{A}, \emph{t})}$ would be the same as $\text{Regret-MAB(\emph{M}, \emph{t})}$, for an appropriately formulated MAB algorithm $\emph{M}$.
 \subsubsection{UCB Algorithm on $-(y_{kj}-(\textbf{w}^\top \textbf{x}_k))^2$}
As mentioned, we can work with MAB algorithms on $-(y_{kj}-(\textbf{w}^\top \textbf{x}_k))^2$ for which we look at the widely used UCB family of MAB algorithms. The UCB algorithm is an index based scheme which, at time instant $t$ selects an arm $i$ that has the maximum value of sum of the estimated mean 
($\hat{\gamma_i}$) and a carefully designed confidence interval $c_{i,t}$ to provide desired guarantees.  
To design the UCB confidence interval $c_{i,t}$, a fairly general class of algorithms called $\psi$-UCB \cite{Bubeck2012a} can be used.
The procedure for applying $\psi$-UCB for a random variable $G$ with some arbitrary distribution,
involves choosing a convex function $\psi_G(\lambda)$, such that, $\ln \mathbb{E}[\exp(|\lambda(G - \mathbb{E}[G])|] 
\leq \psi_G(\lambda)$ for all $\lambda \geq 0$. Further, an application of Chernoff bounds gives the confidence interval. 
In particular when $G$ satisfies the sub-Gaussian property, the choice of $\psi_G(\lambda)$ is easy. 
In our setting, we will see that $\psi$-UCB
is inapplicable.
\begin{lemma}{Inapplicability of $\psi$-UCB:}
\label{lemma:psi-ucb-not-possible}
Let the distribution of random variables $G_j$ follow a zero-mean normal distribution for $j= 1, \cdots, m$. 
The distribution of $-G_j^2$ is sub-exponential 
which is a heavy-tailed distribution. For an MAB framework where the rewards of the 
arms are sampled from $-G_j^2$,  $\psi$-UCB is not applicable.
\end{lemma}

\begin{proof}

A variable $G$ is sub-exponential if $\mathbb{E}[\exp(\lambda G)] \leq 1/(1-\lambda/a)$ for $0 < \lambda < a$.
We now prove that the random variable $G^2$, where $G \sim \mathcal{N}(0,\sigma^2)$ is sub-exponential.
\begin{align}
  \mathbb{E}&[\exp(\lambda G^2)] = 1/(\sigma \sqrt{2 \pi}) \int_{-\infty}^{\infty} \exp(z^2 (\lambda - (1/2\sigma^2))) dz \\
  & = 1/(\sigma \sqrt{2 \pi}) \int_{-\infty}^{\infty} \exp(-z^2 ((1-2\lambda\sigma^2 /2\sigma^2))) dz \\
  &= 1/(\sigma \sqrt{2 \pi}) \int_{-\infty}^{\infty} \exp(-z^2 /(2\sigma^2/(1-2\lambda\sigma^2 ))) dz \\
  & = 1/\sqrt{1-2\lambda\sigma^2} \\
  & < 1/(1-2 \lambda \sigma^2)  \text{ for } 0\leq \lambda < 1/2\sigma^2
\end{align}
Setting $ a = 1/\sigma^2$ shows that $G$ is sub-exponential. A random variable $-G^2$ is sub-exponential iff $G^2$ is sub-exponential.
Therefore $-G^2$ is sub-exponential. 

 Let $G_j \sim \mathcal{N}(0,\sigma_j^2)$. We now compute the functions, $\mathbb{E}[\exp(\lambda(-G_j^2 + \mathbb{E}[G_j^2]))]$ and 
 $\mathbb{E}[\exp(\lambda( \mathbb{E}[-G_j^2] + (G_j^2 )))]$.
 $\mathbb{E}[G_j^2] = \sigma_j^2$.
\begin{align*}
 \mathbb{E}&[\exp(\lambda(-G_j^2 + \mathbb{E}[G_j^2]))] = \mathbb{E}[\exp(\lambda(-G_j^2 + \sigma_j^2))] \\
 &= \frac{\exp(\lambda\sigma_j^2)}{\sigma_j\sqrt{2\pi}} \int_{-\infty}^{\infty} \exp(-\lambda x^2) \exp(\dfrac{-x^2}{2\sigma_j^2}) dx \\
 & =  \exp(\lambda\sigma_j^2)/(\sigma_j\sqrt{2\pi}) \int_{-\infty}^{\infty} \exp(-x^2 / 2(\sigma_j^2/(1+ 2\lambda \sigma_j^2))) dx\\
 &= \exp(\lambda \sigma_j^2)/\sqrt{1 + 2\lambda\sigma_j^2}
\end{align*}
Similar calculations also yield,
\begin{align*}
 \mathbb{E}&[\exp(\lambda(G_j^2 - \mathbb{E}[G_j^2]))] 
 = \exp(-\lambda \sigma_j^2)/\sqrt{1 -2\lambda\sigma_j^2}
\end{align*}

In order to apply $\psi$-UCB for the MAB framework where the rewards of the  arms are sampled from $-G_j^2$, we need to compute a function
$\psi(\lambda)$ such that for all $\lambda \geq 0$,  $ \ln \mathbb{E}[\exp(\lambda(G_j^2 - \mathbb{E}[G_j^2]))] \leq \psi(\lambda) $
and $\ln \mathbb{E}[\exp(\lambda(-G_j^2 + \mathbb{E}[G_j^2]))] \leq \psi(\lambda) $. $ \mathbb{E}[\exp(\lambda(G_j^2 - \mathbb{E}[G_j^2]))]$
is not even defined for $\lambda \geq 1/(2\sigma_j^2)$ and hence the function $\psi(\lambda)$ cannot be computed. Therefore $\psi$-UCB cannot be 
applied to this framework.
\end{proof}
In our setting, $y_{kj} - \textbf{w}^\top \textbf{x}_k$ follows a normal distribution and $-(y_{kj} - \textbf{w}^\top \textbf{x}_k)^2$ 
has a sub-exponential distribution which is heavy tailed. Therefore from Lemma \ref{lemma:psi-ucb-not-possible},
an upper confidence interval cannot be obtained using $\psi$-UCB.

\subsubsection{Robust-UCB with Truncated Empirical Mean}
To devise upper confidence intervals for heavy tailed distributions,
Robust UCB \cite{Bubeck2012b} prescribes working with `robust' estimators such as a truncated empirical mean, 
where samples that lie beyond a carefully chosen range are discarded.
The necessary condition to be satisfied while applying Robust UCB is that the reward distribution of the arms should have
moments of order $1+ \epsilon$ for some $\epsilon \in (0,1]$. 
Since the distribution of $-(y_{kj}-(\textbf{w}^\top \textbf{x}_k))^2$ has finite variance, Robust UCB with the truncated empirical mean
can be used by setting $\epsilon = 1$.
At round $t$, the truncated empirical mean of the samples, 
the absolute value of which do not exceed $\sqrt{ut/\log \delta^{-1}}$, is computed as,
\begin{equation}
\label{eqn:truncated-mean}
\hat{\mu}_{t}^j =\dfrac{1}{n_j^c}\sum_{i:I_{ij}=1 } \xi_{ij} \;\;\; \mathbbm{1}\left(
  |\xi_{ij}| \leq \sqrt{ut/\log \delta^{-1}} \right)
\end{equation}
where $\xi_{ij} = -(y_{ij}- \mu_w^\top \textbf{x}_i)^2$ and 
 $\mu_w$ is the estimator of $\textbf{w}$ obtained from the variational inference algorithm. 
 In Eqn \ref{eqn:truncated-mean}, $n_j^c$ is the number of samples that are actually considered, $\delta$ is the desired confidence on the 
deviation of $\hat{\mu}_{t}^j$ from  $1/\beta_j$ for all $j$, $u$ is an upper bound on $\xi_{ij}^{1 +\epsilon}$. 
From Lemma \ref{lemma:var_upd_aysmptotics} $\mu_w$ is an unbiased estimate for $\textbf{w}$ and hence we use $\mu_w$ 
instead of $\textbf{w}$. The parameter $\delta$ can be tuned appropriately to get 
tight bounds on the regret.
We now describe the algorithm. 
\\\hrule
\begin{algorithm} [h]
 \caption{Robust UCB for selecting the annotators}
 \label{alg:robust-ucb}
 \KwIn{No. of annotators $m$, Unlabeled set $\mathcal{U}$, Labeled set $\mathcal{D}$, $n_j$, $n_j^c$, for $j = 1, \cdots, m$}
 Set $\mu_w, \Lambda_w$ using variational inference procedure described earlier;
 $t := 0 $ \;
 Set $\hat{\mu}_t^j$ for the
 annotators using Eqn (\ref{eqn:truncated-mean})\;
 \While{ ( the learner has budget or the model has not attained the desired RMSE )}{
 \begin{itemize}
     \item Choose an instance $\textbf{x}_k = 
     \argmax_{\textbf{x} \in \mathcal{U}} \textbf{x}^\top \Lambda_w^{-1} \textbf{x}$ ;\\
   \item Get a label $y_{kj^*}$ from an annotator $j^*$ such that $j^* \in \underset{1 \leq j \leq m}{\argmax} ~ \hat{\mu}_{t}^j + \sqrt{32u (\log t)/n_j} $  ;\\
   \item $t :=  t+1$ ; $n_{j^*} :=  n_{j^*} + 1$
   ; $\mathcal{D} :=  \mathcal{D} \cup \{(\textbf{x}_k,  y_{kj^*})\}$ ;
   \item Run variational inference procedure described earlier \\and update $\mu_w$;
   \item If $(y_{kj^*}- \mu_w^\top \textbf{x}_i)^2 < \sqrt{ut/\log \delta^{-1}}$ 
   \begin{itemize}
     \item $n_{j^*}^c := n_{j^*}^c + 1$ ;
     \item Update $ \hat{\mu}_t^{j*}$ using Eqn (\ref{eqn:truncated-mean});
   \end{itemize}
 \end{itemize}
 }
\end{algorithm}
\hrule

\begin{theorem}
\label{thm:regret-seq-A}
  Regret-Seq$(\text{Algo } \ref{alg:robust-ucb},T) \leq \sum_{i:\Delta_i>0} \dfrac{32u \log T}{\Delta_i} + 5 \Delta_i $.
\end{theorem}
\begin{proof}
We first prove that, with probability at least $1-\delta$, 
\begin{align}
\label{robust-estimator-assumption}
\hat{\mu}_{t}^j \leq  (-1/\beta_j) + 4\sqrt{u \log \delta^{-1}/n_j}
\end{align}
Let $C_t= \sqrt{ut/\log \delta^{-1}}$. Let the random variable 
$\xi$ = $-(y_{kj}-(\textbf{w}^\top \textbf{x}_k))^2$. As mentioned earlier $\xi^{1+\epsilon} = \xi^2 < u$. Note that 
\begin{align}
\mathbb{E}\left[\xi^2 
\mathbbm{1}_{ \xi\leq C_t}\right]&=\mathbb{E}\left[|\xi^2| 
\mathbbm{1}_{ \xi\leq C_t}\right] \leq u 
\end{align}
\begin{align}
\label{holder-ineq}
\mathbb{E}[\xi\mathbbm{1}_{|\xi| > C_t }]& \leq \mathbb{E}[|\xi|^2]^{1/2}
\mathbb{E}[|\mathbbm{1}_{ \xi \geq C_t}|^2]^{1/2} \leq \sqrt{u} (P\{\xi \geq C_t\})^{1/2} \nonumber \\
& \leq \sqrt{u}(\mathbb{E}[\xi^2]/C_t^2)^{1/2} = u/C_t  
\end{align}
Equation\nobreakspace \textup {(\ref {holder-ineq})} arises due to Holder's inequality. Further,
\begin{align}
\mathbb{E}&[\xi] - \frac{1}{n_j} \sum_{t=1}^{n_j} \xi_t \mathbbm{1}_{[\xi_t \leq C_t]} = \frac{1}{n_j}\sum_{t=1}^{n_j}(\mathbb{E}[\xi]  - \mathbb{E}[\xi\mathbbm{1}_{|\xi| \leq C_t }]) \nonumber\\
&\;\;\;\;\;\;\;\;\;\; + \frac{1}{n_j}( \mathbb{E}[\xi\mathbbm{1}_{|\xi| \leq C_t }] -  \xi_t \mathbbm{1}_{[\xi_t \leq C_t]}) \nonumber \\
&=  \frac{1}{n_j}\sum_{t=1}^{n_j}\mathbb{E}[\xi\mathbbm{1}_{|\xi| > C_t }] +  \frac{1}{n_j}( \mathbb{E}[\xi\mathbbm{1}_{|\xi| \leq C_t }] -  \xi_t \mathbbm{1}_{[\xi_t \leq C_t]}) \nonumber \\
&\leq \frac{u}{C_t} + \sqrt{\frac{2u \log \delta^{-1}}{n_j}}
 + \frac{2C_n \log \delta^{-1}}{3n_j}
 \label{bernstein-eq}
\end{align}
The first term in Eqn (\ref{bernstein-eq}) arises as a consequence of Eqn (\ref{holder-ineq}) and the remaining terms arise as a result of Bernstein's inequality with some simplification. Further algebraic simplification of Eqn (\ref{bernstein-eq}) gives us Eqn (\ref{robust-estimator-assumption}).
\\For a MAB algorithm $A$ using $\hat{\mu}_t^j$ as an estimator
for $-1/\beta_j$, the regret satisfies the following bound when $\delta = T^{-2}$, where $T$ is the total time horizon of plays of the  MAB algorithm.
\begin{align}
\label{robust-ucb-regret} 
 \text{Regret-MAB(}A,T) \leq \sum_{i:\Delta_i>0} \dfrac{32u \log T}{\Delta_i} + 5 \Delta_i .
\end{align}
 Proof of Eqn (\ref{robust-ucb-regret}) involves bounding the number of trials where a sub-optimal arm is pulled, similar to the technique in \cite{UCBAuer2002,Bubeck2012b}. A pull of a sub-optimal arm indicates one of the following three events occur:(1) The mean corresponding to the best arm is underestimated (2) the mean corresponding to a sub-optimal arm is over-estimated (3) the mean corresponding to the sub-optimal arm is close to that of the optimal arm. Next we bound each of the three events and use union bound to get the final result.  Eqn (\ref{robust-estimator-assumption}) is used to get bounds for events (1) and (2). 
Regret-Seq$(\text{Algo } 1, T) = $ Regret-MAB$(\text{Robust-UCB},T )$ from \cite{Neufeld14}. 
\end{proof}
\begin{theorem}
\label{thm:wasteage}
 The expected number of samples discarded by the Robust UCB algorithm in $t$ trials of the algorithm,
 $\mathbb{E}[W(t)]\leq 4 (\log t)^2$. 
\end{theorem}
\begin{proof}\let\qed\relax
 As per the robust UCB algorithm, at the $t^{th}$ time instant, the probability of the random variable $ \xi = (y_{kj} - \textbf{w}^\top \textbf{x}_k))^2$ exceeding 
 $(ut/(4\log t))^{1/(1+\epsilon)}$,
 \begin{align*}
  P(\xi &> (ut/(4\log t))^{1/(1+\epsilon)}) = P(\xi^{1+\epsilon} > ut/(4 \log t)) \\
  & \leq    \dfrac{\mathbb{E}[\xi]^{1+\epsilon} 4 \log t }{ut} \text{ (by Markov inequality)} \\
  & \leq 4 \log t/t
 \end{align*}
The number of samples discarded upto a time $n$ is 
\begin{align*}
 \mathbb{E}[W(n)]&= \sum_{t=1}^n \mathbb{E}[\mathbbm{1}[Z_t > (ut/(4 \log t))^{\epsilon/(1+\epsilon)}]] \nonumber\\
 & = \sum_{t=1}^n 4 \log t/t \leq 4 (\log n)^2 \;\;\;\hfill \square
\end{align*} 
\end{proof}

\section{The Case of Strategic Annotators}
\label{sec:payment_strategic}
Till now, we have inherently assumed that annotators are non-strategic. Now we look at the scenario where an annotator who has been allocated an instance
 is strategic about how much effort to put in. For this, we assume that, for each annotator $j$, the precision $\beta_j$ introduced while labeling an instance is proportional to the effort put in by annotator $j$.
We now refer to the effort as $\beta_j$ for simplicity. 
It is best for the learning algorithm when the annotator $j$ puts in as much effort (high $\beta_j$) as possible thereby reducing the
variance in the labeled data. A given level of effort incurs a cost to the annotator $c_j(\beta_j)$. We assume that $c_j(.)$ is a non-negative strictly increasing function of $\beta_j$ with $c_j(0) = 0$. The exact form of $c_j(.)$ is unknown to the learner.
From the annotator's point of view, a high value of effort $\beta_j$ might incur a higher cost and thus 
the annotator might not be motivated to put 
in higher effort.
\par 
In order to take into account the strategic play of the human annotators, we appeal to mechanism design techniques. Mechanism design comprises allocation 
and payment rules. 
The mechanism is to be designed to meet at least the following objectives.
\begin{definition}{ Individual Rationality (IR):}
 A mechanism is IR if the expected utility of every participating agent is non-negative.
\end{definition}
\begin{definition}{ Quality Compatibility:}
We say a mechanism is `quality compatible' at level
$\underline{\beta}$ if it induces every participating agent 
to operate under precision $\beta \geq \underline{\beta}$.
\end{definition}
We now present a mechanism design solution which meets the above design goals.\\
\textbf{Proposed Mechanism:}
 (1) We use Algorithm \ref{alg:robust-ucb} as the allocation rule. 
 (2)The payment rule for annotator $j$ when his estimated precision is $\hat{\beta}_j$ is,  
 \begin{align}
    P(\hat{\beta}_j)= B \min \left\lbrace 1, \max \left\lbrace 0,  \left( \frac{\hat{\beta}_j -\underline{\beta}}{\overline{\beta} -\underline{\beta}} \right)\right \rbrace \right \rbrace
 \end{align}
 We assume that the learner has a finite budget $B$ per example.
Also the annotators are expected to have precisions in the
range $[\underline{\beta}, \overline{\beta} ]$. An effort level in this expected range fetches a corresponding proportional payment to the annotator. If an annotator puts in 
an effort less than $\underline{\beta}$, he does not receive any payment.  An effort level higher than $\overline{\beta}$ fetches an annotator a maximum
payment of $B$, due to the limitation on the willingness of the learner.

\textbf{Annotator's optimization problem:}
The utility of the annotator when operating at the effort level $\beta_j$ is 
$U(\beta_j) =  P(\beta_j) - c_j (\beta_j)$. The optimal effort for the annotator, 
$\beta_j^* = \underset{\beta_j}{\text{argmax}} \quad U(\beta_j)  $.
\begin{theorem}{}
\label{thm:ir-ic-payment}
 The proposed mechanism is IR and quality compatible. 
\end{theorem}
\begin{proof}\let\qed\relax
\begin{figure}
\centering
   \includegraphics[scale=0.55]{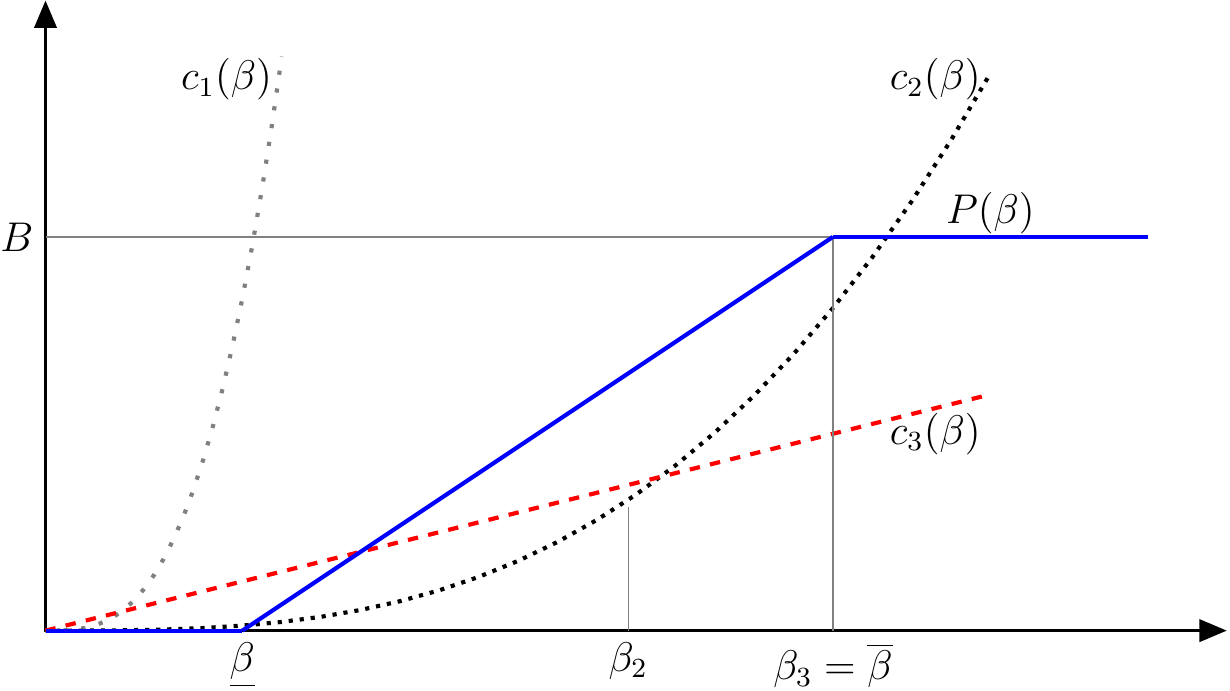}
    \caption{Realizations of $c_j(\beta_j)$ in relation to the payment rule $P(\beta_j)$  }
   \label{fig:ic-proof}
\end{figure}
The payment scheme is individually rational as annotators participate only when $P(\beta_j) > c_j(\beta_j)$ in that case, they obtain a positive utility.
The utility is therefore non-negative and hence the mechanism is IR.

In order to prove that the payment scheme is quality compatible, we consider the
three possible realizations of $c_j(\beta_j)$ in relation to the payment rule $P(\beta_j)$.
\begin{enumerate}
 \item There exists no $\beta_j$ for which $P(\beta_j) > c_j(\beta_j)$. In this scenario, an annotator will choose to not participate,
 as there is clearly no benefit from participation. The cost function $c_1(\beta)$ in Figure\nobreakspace \ref {fig:ic-proof} captures this.
 \item There exists some $\beta_j$ such that $ \underline{\beta}< \beta_j \leq \overline{\beta}$, for which $P(\beta_j) > c_j(\beta_j)$ and the maximum utility 
 is attained at $\beta_{j^*}$, such that, $ \underline{\beta} <\beta_{j^*} < \overline{\beta}$. The cost function $c_2(\beta)$ in  Figure\nobreakspace \ref {fig:ic-proof} demonstrates this scenario where an effort $\beta_2 > \underline{\beta}$ maximizes his utility. 
 \item There exists $\beta_j$ such that $ \underline{\beta} < \beta_j\leq \overline{\beta}$, for which  $P(\beta_j) > c_j(\beta_j)$ and the maximum utility 
 is attained at $\beta_{j^*} = \overline{\beta}$.
 The cost function $c_3(\beta)$ in  Figure\nobreakspace \ref {fig:ic-proof} demonstrates this.\hfill$\square$
\end{enumerate}
\end{proof}

 \section{Experimental Results}
 \label{sec:experiments}

 We conducted experiments on three real world datasets from the UCI repository  \cite{Lichman:2013} - Housing, Redwine and Whitewine, the details of which are provided in Table \ref{tab:dataset-detail}.
To simulate the annotators, we added zero-mean Gaussian noise to the output variables.
$1/\sqrt{\beta_j}$ values of the annotators were randomly chosen from two sets of intervals
$U1 = [0.1, 1]$ and $U2 = [1,2]$. 
Annotators with $1/\sqrt{\beta_j}$  chosen from interval $U1$ are clearly better than those chosen from $U2$.

     \begin{figure}[t]
      \begin{subfigure}{0.49\linewidth}
        \includegraphics[height=3cm]{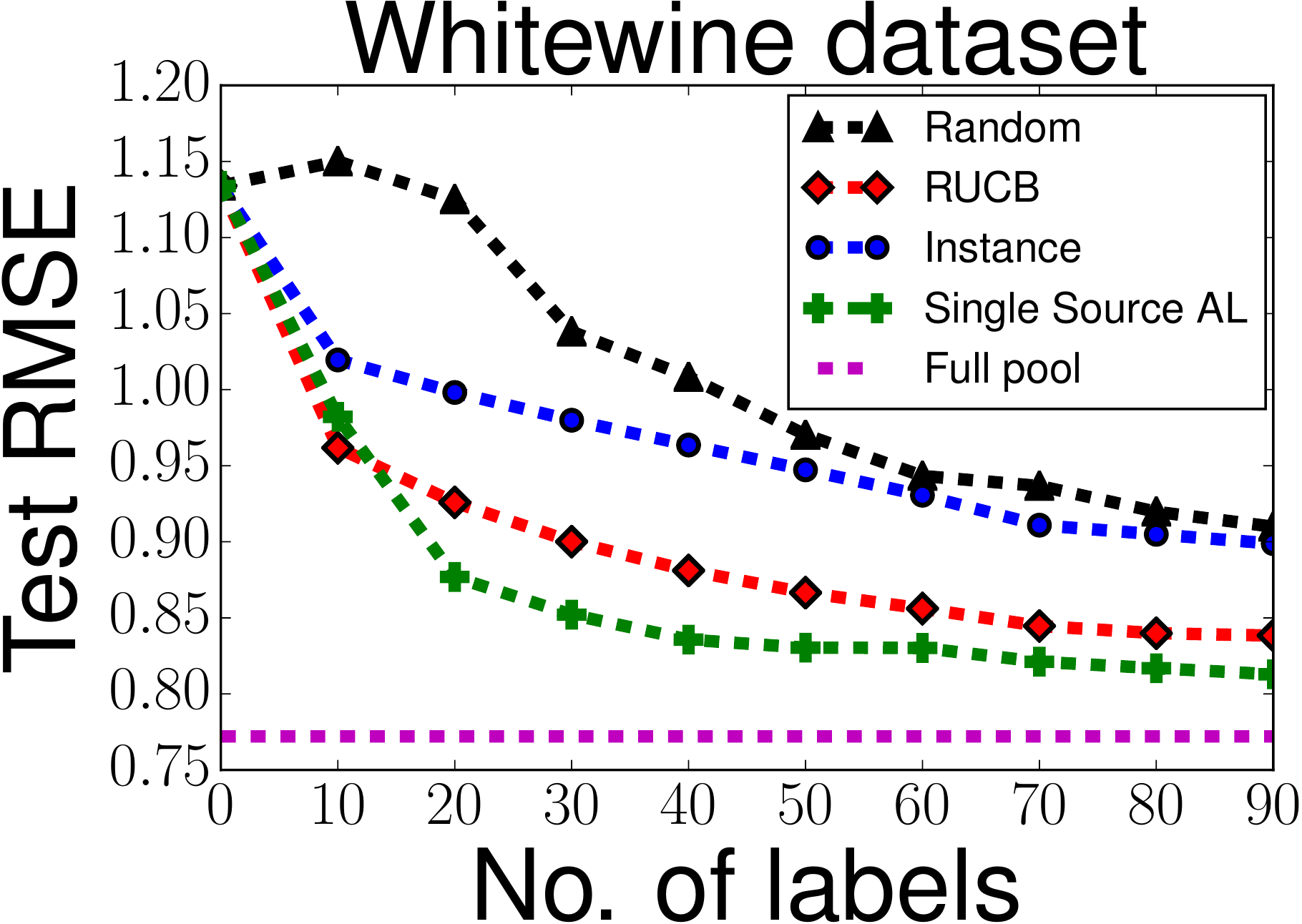}
          \subcaption{ }
        \label{fig:whitewine_rmse}
            \end{subfigure}
              \begin{subfigure}{0.49\linewidth}
        \includegraphics[height=3cm]{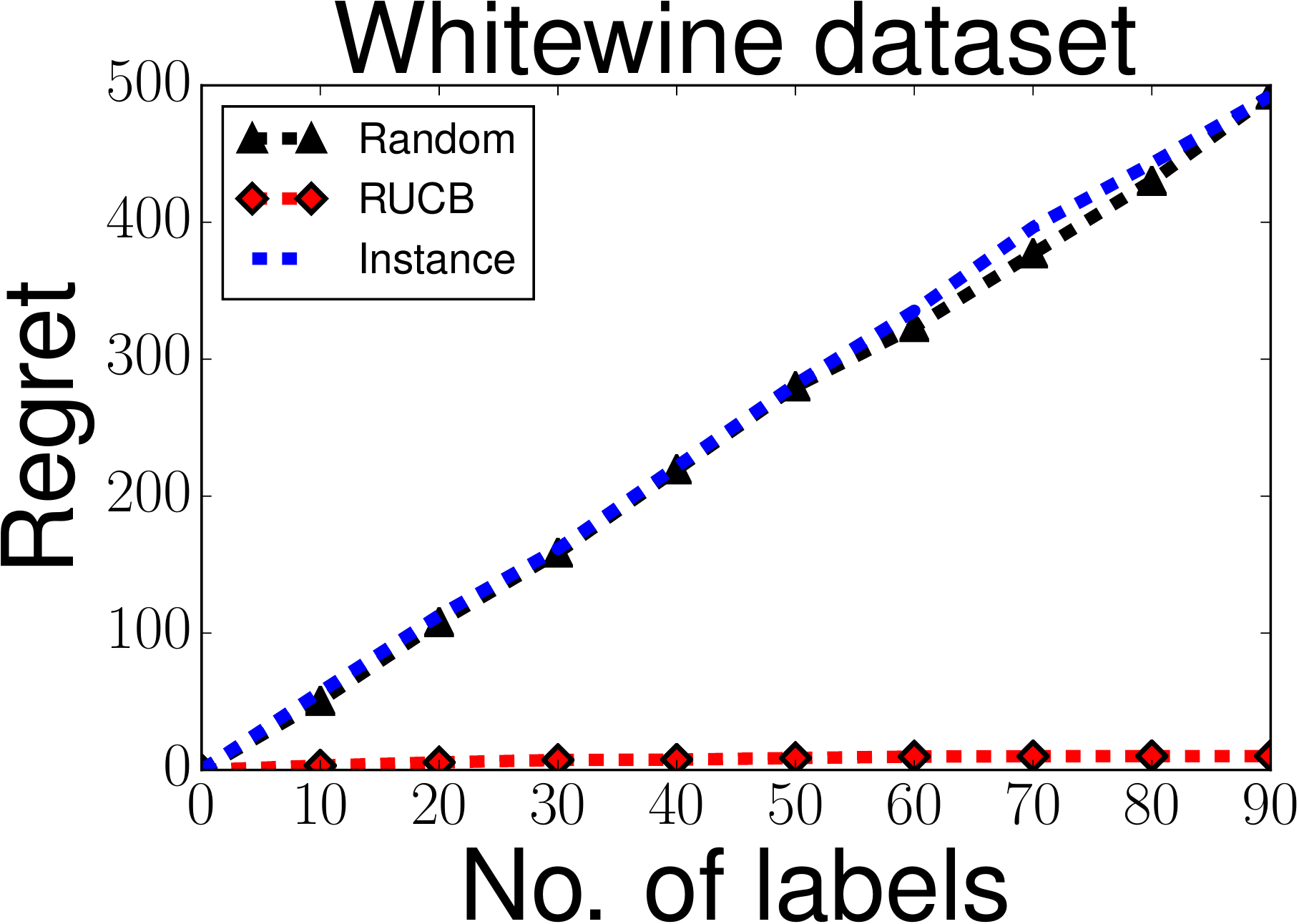}
        \subcaption{ }
        \label{fig:payment_whitewine_regret}
    \end{subfigure} \\    
    \begin{subfigure}{0.49\linewidth}
       \includegraphics[height=3cm]{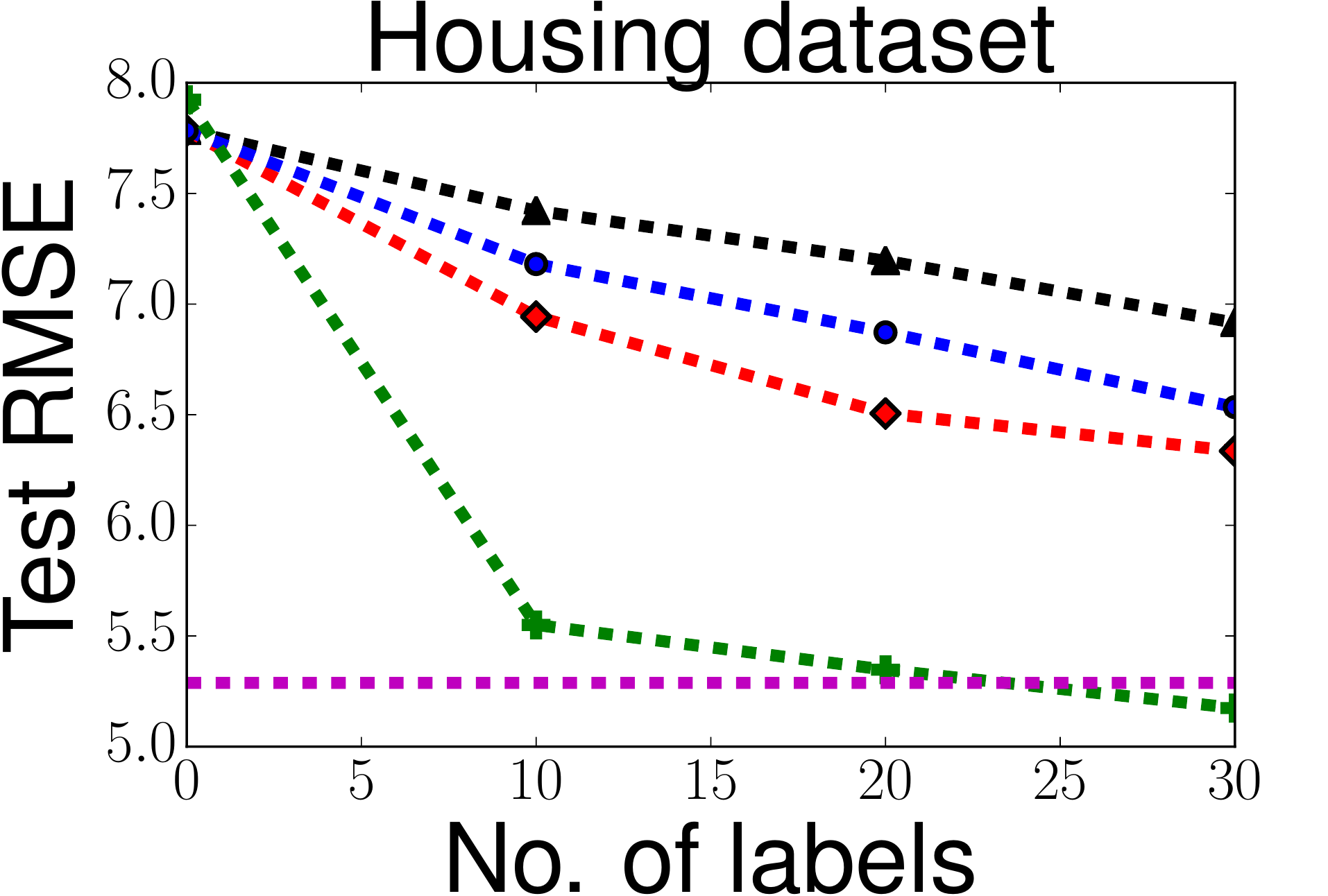}
       \subcaption{ }
        \label{fig:housing_rmse}
    \end{subfigure}
        \begin{subfigure}{0.49\linewidth}
        \includegraphics[height=3cm]{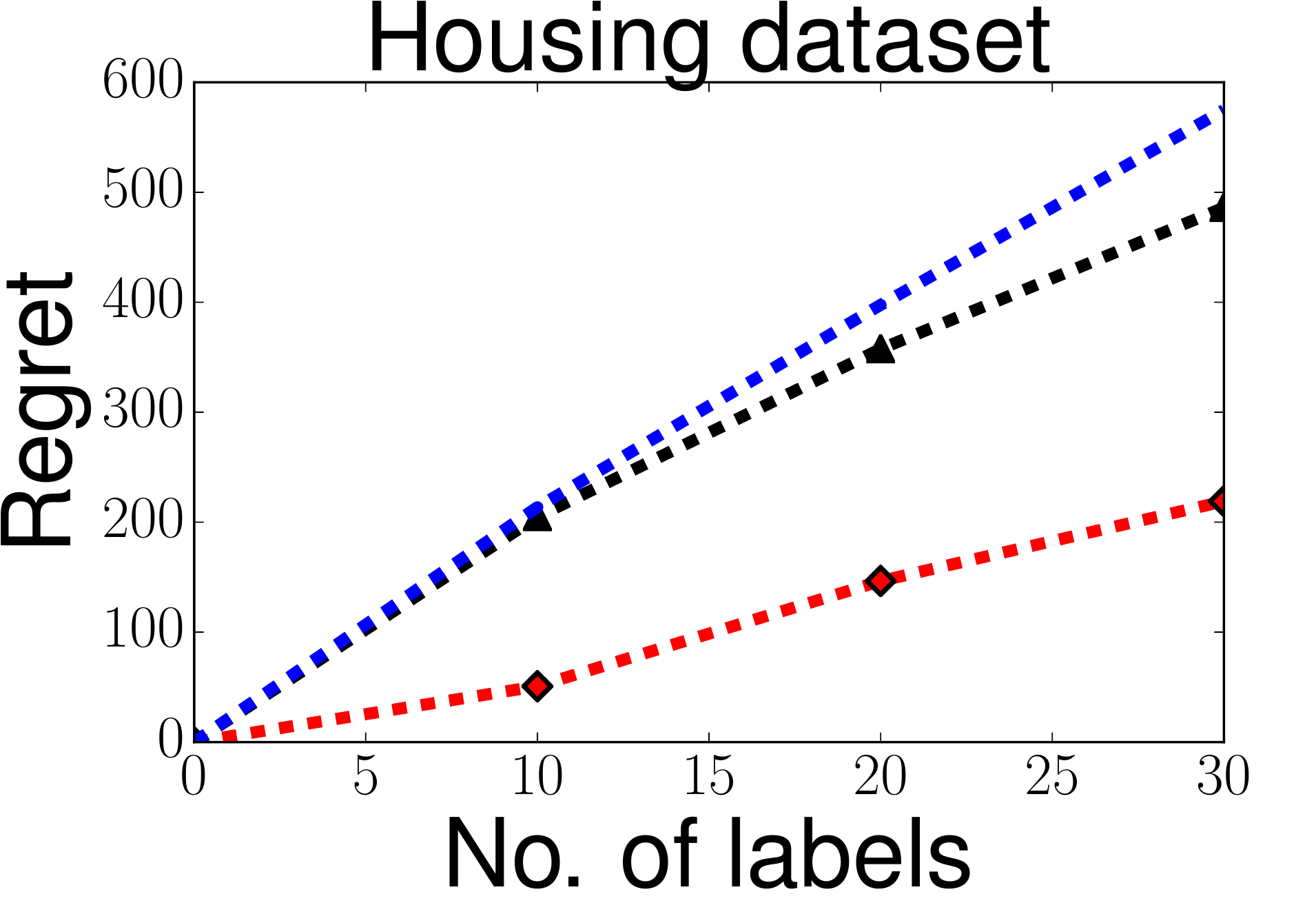}
        \subcaption{ }
         \label{fig:payment_housing_regret}
    \end{subfigure}%
    \\
         \begin{subfigure}{0.49\linewidth}
         \includegraphics[height=3cm]{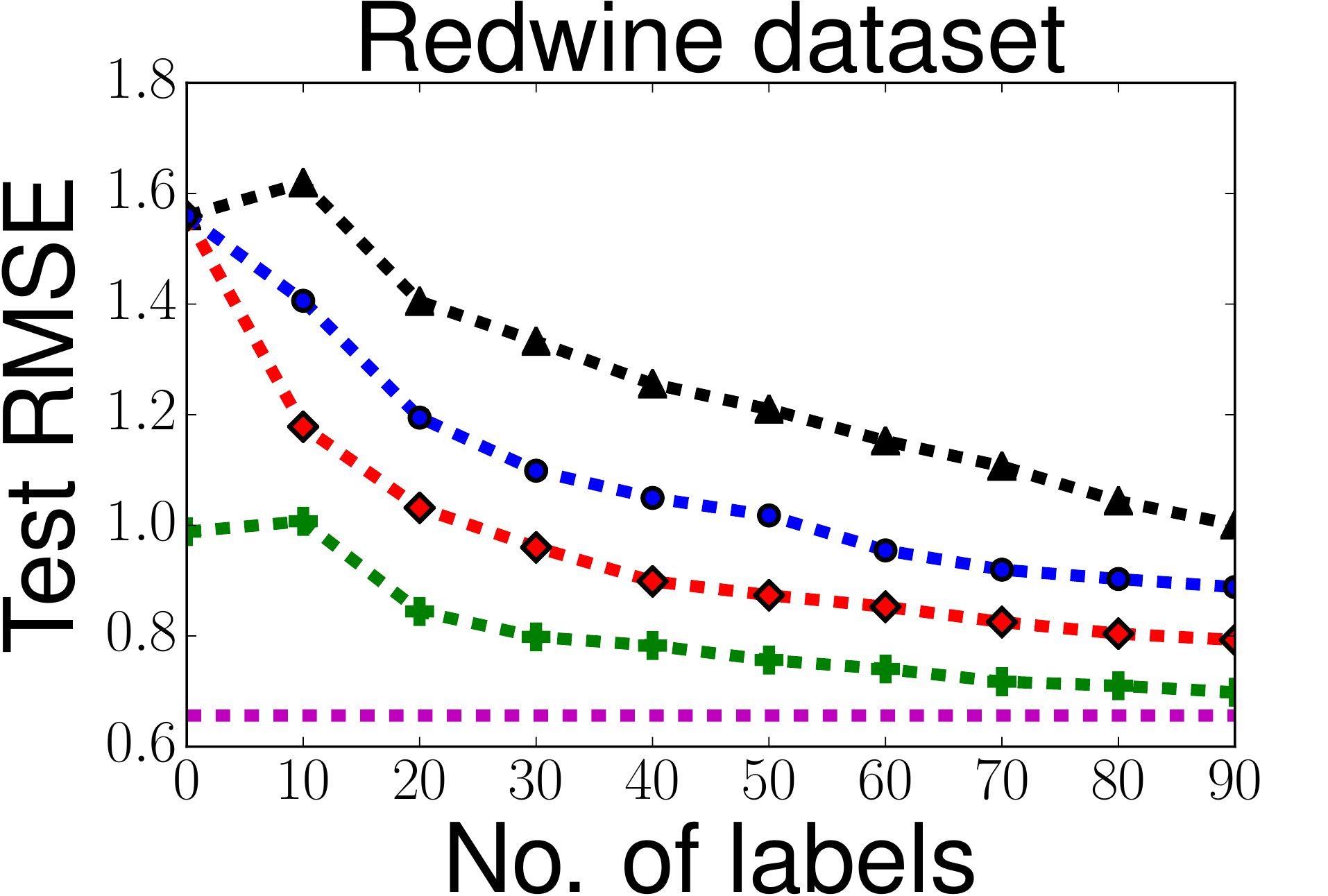}
         \subcaption{ }
         \label{fig:redwine_rmse}
     \end{subfigure}
       \begin{subfigure}{0.49\linewidth}
         \includegraphics[height=3cm]{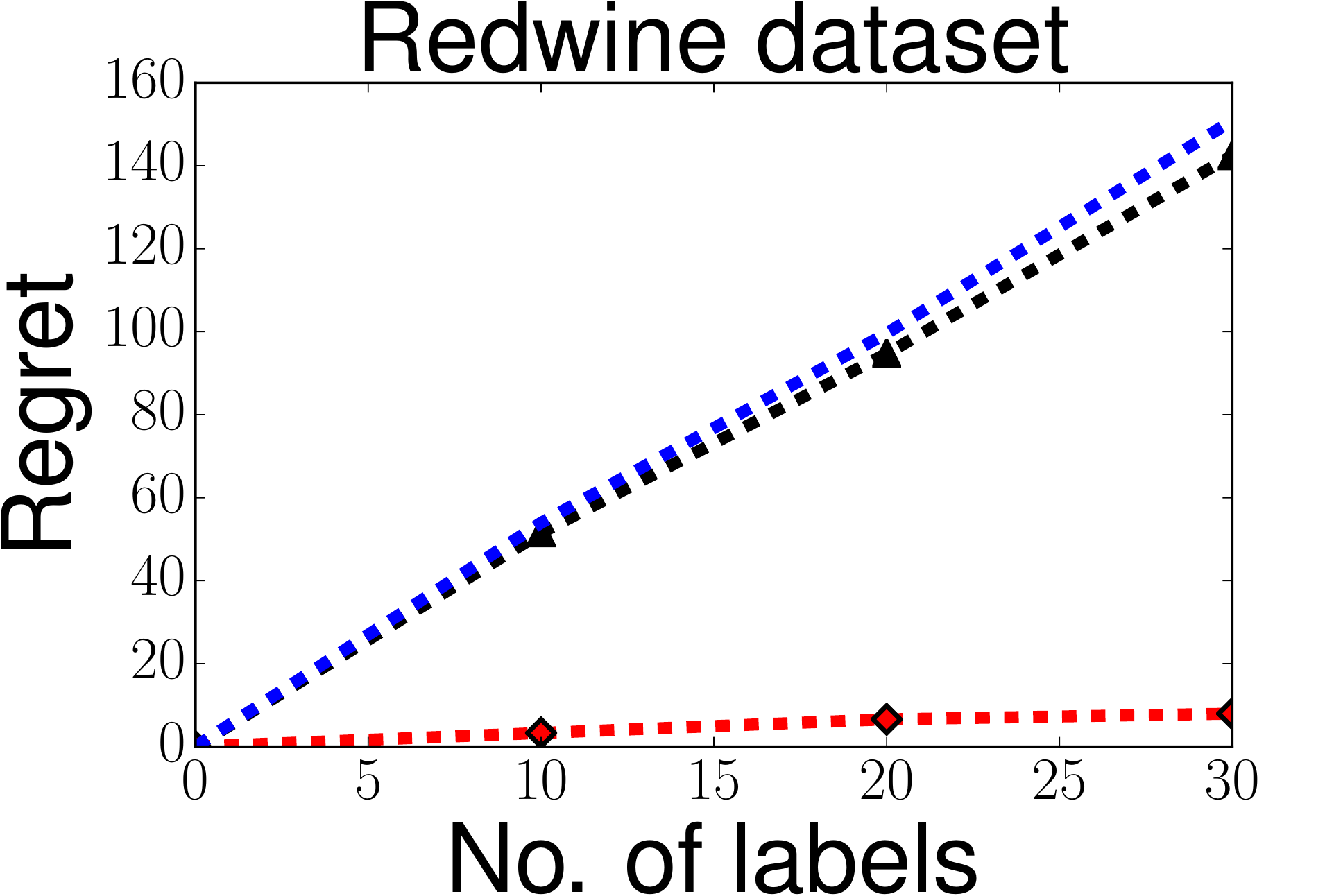}
 \subcaption{ }
         \label{fig:payment_redwine_regret}
         \end{subfigure}    
\vspace{-3mm}
    \caption{Active learning results on various datasets.
 The legends for the figures in
    each row are provided on the corresponding figure in the first row}
    \label{fig:al-plots}
    \vspace{-5mm}
    \end{figure}
    
    \subsection{Data Preprocessing}
We worked with a transformation $\Phi: {\mathbb{R}}^d \rightarrow {\mathbb{R}}^d$ of the original data matrix $\textbf{X}$. For the Housing and the Whitewine datasets, we worked with the following non-linear transformation
$
 \label{sigmoid-transform}
 \Phi_b(\textbf{x}; R_b, s) = 1/(1+\exp(- \left \Vert\textbf{x} - R_b\right \Vert/s))
$, $b = 1 \rightarrow d $
whereas, for the Redwine dataset, the
original data matrix $\bold{X}$ was used.
The value of $s$ was fixed using cross-validation.
The parameters $R_b$ for $b = 1 \rightarrow d$,  were set as the $k$-means cluster representatives of the dataset.
All the features were normalized.

 \subsection{Performance of Bayesian Parametric Model}
We compared our Bayesian parameter estimation algorithm (without active learning) with
 MLE  \cite{Raykar2012JMLR} and Gaussian Process based method  \cite{groot}. From
 the complete dataset $\mathcal{C}$, a random $30\%$ of data was used as test dataset $\mathcal{T}$. 
 We refer to the set $\mathcal{C} \setminus \mathcal{T}$ as the
 full pool of training instances $\mathcal{F}$. 50 annotators were used, out of which,  for $40$ of them,
 the parameter $1/\sqrt{\beta_j}$ was chosen from $U1$,
and for the remaining, from $U2$. The parameters of the Bayesian model described earlier 
were learnt
 using the full pool $\mathcal{F}$ labeled by the 50 annotators, as the training data.
 The experiments were repeated with 10 different splits of the data.
 We report the Average Root Mean Square Error (RMSE) scores on the test set $\mathcal{T}$. The RMSE for the test dataset containing
 $N_\text{test}$ instances
 with true output vector $\bold{z}$ and predicted output vector $\hat{\bold{y}}$, is calculated as,
$
  \text{RMSE}(\hat{\bold{y}};\bold{z}) = \sqrt{ \sum\nolimits_{i=1}^{N_\text{test}} \left(\hat{\bold{y}}_i -
  \bold{z}_i\right)^2 /N_\text{test}}
$.
Our results are provided in Table \ref{tab:var_inf_results_real_bad_annot_more_data}. Our method consistently outperforms Groot's method and compares well with MLE. 
\begin{remark}
With increasing size of the dataset, the performance of our model approaches MLE (as demonstrated in Table\nobreakspace \ref {tab:var_inf_results_real_bad_annot_more_data}, Whitewine dataset). This is consistent with the result that with increased size of training data set, Bayesian estimates perform similar to MLE \cite{BishopBook}. It further shows the efficacy of our learning scheme explained in  Section\nobreakspace \ref {sec:bayesian-lr}. The additional advantage that our model offers is the suitability to further apply active learning methods, which is not offered by other learning schemes like MLE\cite{Raykar2012JMLR} and Groot et al \cite{Ristovski2010}. 
\end{remark}

 \begin{table}[ht]

 \begin{subtable}{0.48\textwidth}
  \centering
  \begin{tabular}{|l|l|l|l|}
  \hline  Dataset & Size & $d$ &  $\Phi$ \\\hline
  Housing & 506  & 12 & Nonlinear \\
  Redwine &  1599 & 11 & Linear \\
  Whitewine & 4898 & 11 & Nonlinear \\
   \hline
  \end{tabular}
   \subcaption{Details of datasets}
\label{tab:dataset-detail}
 \end{subtable}
\begin{subtable}{0.48\textwidth}


\centering
  \begin{tabular}{|l|l|l|l|}
  \hline  Dataset & Our method & MLE & Groot et al.\\\hline
  Housing & \textbf{4.7209} & 4.93834  & 5.998169 \\
  Redwine & \textbf{0.51490}  & 0.65868	 & 0.67354 \\
  Whitewine & \textbf{0.75740} & 0.75748	&  1.235 \\
   \hline
  \end{tabular}

 \subcaption{Average test RMSE values when the whole dataset is used (without active learning). `Our method' refers to the variational inference based learning scheme explained in Section\nobreakspace \ref {sec:bayesian-lr}}
   \label{tab:var_inf_results_real_bad_annot_more_data}

 \end{subtable}
 \vspace{-3mm}
  \caption{Details of datasets and performance of the model}
  \vspace{-2mm}
 \end{table}

\subsection{Active Learning Experiments}
We now describe our experiments with the active learning criteria.
\label{sec:AL_expts}
 In order to test the results of active learning on linear regression, we used the set $\mathcal{T}$ as the test dataset as in the previous case.
  Initially, only 10 instances from $\mathcal{F}$ labeled by all annotators were used as the training set $\mathcal{D}$.
 $\mathcal{F} \setminus \mathcal{D}$ was used as the unlabeled
  set $\mathcal{U}$.   
 At every step of active learning, the label $y_{kj^*}$ of one instance $\textbf{x}_k$ was procured from an annotator $j^*$,
 chosen using Algorithm \ref{alg:robust-ucb}. The model was relearnt using the new training set
 $\mathcal{D} = \mathcal{D} \cup \{(\textbf{x}_k, y_{kj^*}) \}$.
 The RMSE was calculated on $\mathcal{T}$ and the results were  plotted at every step. We also plotted the regret 
 for Algorithm \ref{alg:robust-ucb} at 
 every step.
 The experiments were repeated for 10 different splits of the dataset. The test RMSE when the set $\mathcal{F}$ was used for training (so that $\mathcal{D} = \mathcal{F}$) was also plotted.  This error is the best achievable error in the crowdsourcing scenario. \\
 To the best of our knowledge, our work is the first attempt towards active learning for regression from the crowd and therefore there are no other baselines in the literature to compare our method against.
 However we have used the following baselines for comparison:\\
(1)Random: Random selection of instances and annotators.\\
(2)Instance: Algorithm \ref{alg:robust-ucb} for selecting the instances and random selection of annotators.\\
(3)Single Source AL: The labels were provided by a single source with negligible noise. Active selection of instances was performed using uncertainty sampling.  
 The RMSE and regret plots are provided in Figure \ref{fig:al-plots}. 
Clearly the Robust UCB strategy outperforms `Random' as well as `Instance' with respect to RMSE as well as regret
 and approaches the `Single Source AL' with fewer number of labeled examples.  

 \begin{remark}
 Our active learning algorithm demonstrates a superior performance with just a few additional labels (Figure\nobreakspace \ref {fig:al-plots}).
A similar trend was observed for the rest of the curve, which was omitted in the plots for the sake of clarity. 
 \end{remark}
 
\section{Conclusions and Future Work}
We set up a Bayesian framework to infer the parameters of linear regression using crowds. As closed form Bayesian solution is intractable, we used approximation schemes. 
To improve this initially learnt regression model, we used various active learning techniques 
and studied their theoretical foundations. We  established the connections with MAB algorithms and  explored the use of
Robust UCB for annotator selection in active learning, providing theoretical guarantees and also performing a wastage analysis.
Next, we introduced a payment scheme for annotators to
 ensure that they put in their best efforts while labeling the data. 
Our experiments on real data show the efficacy of our techniques.   
\par
Our approach of Bayesian learning, MAB algorithm for annotator selection, uncertainty sampling for instance selection and design of quality compatible mechanisms to elicit best efforts from crowd workers is applicable for a wide range of tasks like classification, ordinal regression etc. It would be interesting to study the suitability of various MAB algorithms depending on the form of the distributions used to model the annotators' qualities.
Modeling the subjectivity of the annotators, their dynamic entry and exit, and the design of incentives in these scenarios is also challenging.

%
\newpage
\bibliographystyle{abbrvnat}
\bibliography{IncentiveALRegression_V15}  
%
\end{document}